\documentclass[conference]{IEEEtran}
\ifCLASSINFOpdf
\else
\fi
\usepackage{commands}

\begin{document}
\title{Best Arm Identification under Additive Transfer Bandits}

\author{\IEEEauthorblockN{Ojash Neopane}
\IEEEauthorblockA{Machine Learning Department\\
Carnegie Mellon University\\
Pittsburgh, PA\\
Email: oneopane@andrew.cmu.edu}
\and
\IEEEauthorblockN{Aaditya Ramdas}
\IEEEauthorblockA{Department of Statistics and Data Science\\Machine Learning Deparment\\
Carnegie Mellon University\\
Pittsburgh, PA\\
Email: aramdas@cmu.edu}
\and
\IEEEauthorblockN{Aarti Singh}
\IEEEauthorblockA{Machine Learning Department\\
Carnegie Mellon University\\
Pittsburgh, PA\\
Email: aartisingh@cmu.edu}}


%Telephone: (800) 555--1212, Fax: (888) 555--1212}


\maketitle

\begin{abstract}
We consider a variant of the best arm identification  (BAI) problem in multi-armed bandits (MAB) in which there are two sets of arms (source and target), and the objective is to determine the best target arm while only pulling source arms.
In this paper, we study the setting when, despite the means being unknown, there is a known additive relationship between the source and target MAB instances.
We show how our framework covers a range of previously studied pure exploration problems and additionally captures new problems.
We propose and theoretically analyze an LUCB-style algorithm to identify an $\epsilon$-optimal target arm with high probability.
Our theoretical analysis highlights aspects of this transfer learning problem that do not arise in the typical BAI setup, and yet recover the LUCB algorithm for single domain BAI as a special case.
% Finally, we show through simulations that our proposed algorithm outperforms existing algorithms for our general setting and is also competitive when compared to specialized algorithms.
\end{abstract}


% \fi
\IEEEpeerreviewmaketitle

\section{Introduction}
In this work, we study a problem at the intersection of transfer learning and sequential decision making.
At a high-level, the problem we study involves two multi-armed bandit (MAB) instances, which we call the \emph{source} and \emph{target} instances, as well a \emph{transfer function}, which is a known relationship between the two MAB instances. 
Within this setup we define and consider an appropriately modified variant of the $(\epsilon, \delta)$-correct best arm identification (BAI) objective \cite{bubeck2009pure, audibert2010best}. 

\paragraph{Some Motivating Examples}
% \textcolor{red}{remove next sentence, we should only talk about this in related work}
% While a specific instance of the problem we propose has been studied in the context of Monte Carlo tree search \cite{huang2017structured}, we argue that a rigorous development of this transfer learning framework will lead to insights as to how we can develop theoretically sound  algorithms for interesting sequential decision making problems such as Reinforcement Learning. 
We start off by highlighting various scenarios where the need to transfer knowledge between sequential decision making problems arise:
\begin{itemize}
    \item \textbf{Clinical Trials.} 
The first scenario we consider is the application of MABs to clinical trials \cite{kazerouni2019best}.
In this context, the arms can be thought of as the different treatments and we wish to determine which is most effective.
A standard practice in this setup is to test treatments on animals before transitioning to clinical trials for humans.
Ideally, we wish to identify the optimal treatments for humans by only testing the treatments on animals.
Here, we can view the animal trials as the source domain, and human trials as the target domain.

\item \textbf{Sim-to-Real Transfer in Reinforcement Learning.}
A popular paradigm for `cheap' reinforcement learning is sim-to-real transfer in reinforcement learning \cite{bousmalis2018using, rusu2016sim, sadeghi2016cad2rl}.
In the sim-to-real problem, the objective is to learn a robot's control policy for the real world (target domain) while restricting training to computer simulations (source domain).
Currently, in the sim-to-real literature, most algorithms rely on heuristics to learn these control policies -- typically by ensuring that a sufficiently diverse set of environments are encountered during training.
While some of these heuristics have proven to be successful, our theoretical understanding of this problem remains in its infancy.
We believe that studying our proposed problem is a first step towards gaining a better understanding of how to transfer knowledge in more complicated sequential decision making problems.

\item \textbf{Rate adaptation in wireless networks.} Rate allocation in wireless networks has been posed as a bandit optimization problem under fixed channel conditions \cite{combes18rate,qi19rate}. However, it is important to adapt the rate allocation according to varying channel conditions by transferring rate allocation policies between related channel conditions. 

%\item \textbf{Property Testing.} In many applications, it is of interest to identify which set of variables satisfy a given property that is not directly observable, while making few measurements of those variables. 
%For example, identifying the most used spectral band while recording minimal time measurements from sensors to minimize communication. 
%identifying which pollution sensors have average readings above a prescribed threshold. 

\end{itemize}

\paragraph{Paper Outline}
The rest of this paper is organized as follows.
In Section~\ref{sec:problem_setup} we formally define the additive-transfer BAI problem as well as natural notions of correctness.
We cover related work in Section~\ref{sec:related-work}.
Next, in Section~\ref{sec:algorithm} we describe the T-LUCB algorithm for the additive-transfer BAI problem. 
Then, in Section~\ref{sec:results} we provide results on our theoretical analysis of the T-LUCB algorithm.
Finally, in Section \ref{sec:conclusion} we discuss additional relevant work and touch on possible interesting future directions. 
Proofs of all results can be found in the Appendix.

%%%%%%%%%%%%%%%%%%%%%%%%%%%%%%%%%%%%%%%%%%%%%%%%%%%%%%%%%%%%%%%%%%%%%%%%%%%%%%%%%%%%%%%%%%%%%%%%%%%%%%%%%%%%%%%%%%%%%%%%%%%%%%%%%%%%%%%%%%%%%%%%%%%%%%%%%%%%%%%%%%%%%%%%%%%%%%%%%%%%    
%%%%%%%%%%%%%%%%%%%%%%%%%%%%%%%%%%%%%%%%%%%%%%%%%%%%%%%%%%%%%%%%%%%%%%%%%%%%%%%%%%%%%%%%%%%%%%%%%%%%%%%%%%%%%%%%%%%%%%%%%%%%%%%%%%%%%%%%%%%%%%%%%%%%%%%%%%%%%%%%%%%%%%%%%%%%%%%%%%%%    
%%%%%%%%%%%%%%%%%%%%%%%%%%%%%%%%%%%%%%%%%%%%%%%%%%%%%%%%%%%%%%%%%%%%%%%%%%%%%%%%%%%%%%%%%%%%%%%%%%%%%%%%%%%%%%%%%%%%%%%%%%%%%%%%%%%%%%%%%%%%%%%%%%%%%%%%%%%%%%%%%%%%%%%%%%%%%%%%%%%%
\section{Problem Setup}\label{sec:problem_setup}
Before introducing the transfer BAI problem, we briefly review the $\epsilon$-BAI problem within the classical MAB framework. 
    In our notation, we define an $n$-armed MAB instance to be a set of $n$ tuples $\braket{\parenthesis{P_i, \mu_i}}_{i = 1}^{n}$ where $P_i \in \varProbSet$ is a probability distribution in some known set $\varProbSet$ and $\mu_i \coloneqq \bbE_{P_i}[X]$ is the mean of $P_i$. 
    For example, $\varProbSet$ could be the set of all sub-Gaussian distributions.
    In this setup, an algorithm interacts with the MAB instance through a round-based protocol.
    In each rounds, $\round$, the learner selects an arm $I_t \in \{1, \ldots, n\}$, and observes a sample $X_t \sim P_{I_t}$.
    For the $\epsilon$-BAI problem, the objective is to identify an $\epsilon$-optimal arm $\varSelectionRule$ satisfying $\sourceMean_{\varSelectionRule} + \epsilon \geq \max_{i \in [n]} \sourceMean_i $, where $[n] = \{1, \ldots, n\}$. 
    
    This problem is often studied in the so-called \emph{fixed-confidence} setting in which a confidence parameter $\delta$ is given and an algorithm is said to be correct if, with probability greater than $1 - \delta$, it stops and returns an $\epsilon$-optimal arm.
    % For any fixed MAB instance, an algorithms sample-complexity is judged either by the expected number of samples or a high-probability bound on the number of samples required to identify the best arm.
    For any fixed MAB instance, an algorithm's performance is then judged by either a \emph{high-probability} or an \emph{in expectation} upper-bound on the number of samples required to identify an $\epsilon$-optimal arm.
    \emph{In this work, we will give a high probability bound for a variant of the fixed-confidence setting that naturally arises in our setup.}
     
    \tb{Transfer Best Arm Identification.}
    We are now ready to introduce the transfer BAI problem which can be stated as a tuple $(\{\sourceArm_\sourceIndex, \sourceMean_\sourceIndex\}_{\sourceIndex = 1}^{\numSourceArms}, \{\targetArm_\targetIndex, \targetMean_\targetIndex\}_{\targetIndex = 1}^{\numTargetArms}, \varTransferFunction)$.
    Here, $\{\sourceArm_\sourceIndex, \sourceMean_\sourceIndex\}_{\sourceIndex = 1}^{\numSourceArms}$ and $\{\targetArm_\targetIndex, \targetMean_\targetIndex\}_{\targetIndex = 1}^{\numTargetArms}$ are $\numSourceArms$ and $\numTargetArms$-armed MAB instances which we respectively call the \emph{source} and \emph{target} MAB instances and $\functionDef{\varTransferFunction}{\bbR^{\numSourceArms}}{(\bbR^{+})^{\numTargetArms}}$ is a \emph{known} multivariate function which we call the \emph{transfer function}.
    Here, we have written $\bbR^+ \coloneqq \bbR \cup \{\infty, -\infty\}$ to denote the extended real numbers. 
    Specifically, $\varTransferFunction$ relates the means of the target and sources arms in the  sense that 
    \[\targetMean = \varTransferFunction(\sourceMean),\] where $\sourceMean = (\sourceMean_1, \ldots, \sourceMean_\numSourceArms)$ and $\targetMean = (\targetMean_1, \ldots, \targetMean_\numTargetArms)$ refer to the vector of means for the source and target MAB instances.
    In this paper we study the special setting in which $\varTransferFunction$ is an additive function satisfying \[\targetMean_\targetIndex = \varTransferFunction_\targetIndex(\sourceMean) = \sum_{\sourceIndex = 1}^{\numSourceArms} \varTransferFunction_{\targetIndex, \sourceIndex}(\sourceMean_\sourceIndex).\]
    Here, and in the rest of this paper, $\sourceIndex$ will always be used to index source arms, and unless otherwise specified, $\targetIndex$ will be used to index target arms.
    As we discuss more in Section~\ref{subsec:special-settings}, this additive setting is already interesting as it captures a large number of existing problems in addition to introducing new problems.
    To provide more concrete intuition about our algorithm and sample complexity analysis, we will use two running examples: property testing and linear transfer functions.
    
     \tb{Property Testing.} In the property testing problem we are interested in identifying all arms $\sourceIndex \in [\numSourceArms]$ which satisfy some property $\sourceMean_\sourceIndex \in \propertySet_\sourceIndex \subset \bbR$. Our additive transfer framework is able to capture this problem. To do so, we first define 
        \begin{equation}
            \bbI_\propertySet(\mu) =     
            \begin{cases}
                1 & \mu \in \propertySet, \\
                -\infty & \mu \not\in \propertySet.
            \end{cases}
        \end{equation}
        Then for each set $\armSet \in 2^{[n]}$ we define a target arm whose mean is $\targetMean_\armSet = \sum_{\sourceIndex \in \armSet} \bbI_{\propertySet_\sourceIndex}(\mu_\sourceIndex)$.
        Clearly, the optimal target arm will be a function of all source arms for which $\sourceMean_\sourceIndex \in \propertySet_\sourceIndex$.
        We note that whenever we refer to the property testing problem, we will index the target arms with $\armSet$ instead of $\targetIndex$. 
        Additionally, for the property testing problem, we require $\epsilon = 0$. 

    \tb{Linear Transfer Functions.} Another useful special case for contextualizing our results is the setting where the transfer function is a linear transformation of the source means, so that 
    \begin{equation*}
        \targetMean_\targetIndex =  \sum_{\sourceIndex = 1}^{\numSourceArms} \varLinearTransformation_{\targetIndex, \sourceIndex} \sourceMean_\sourceIndex.
    \end{equation*}

    In our proposed framework, we restrict our ability to sample from the target arms, and only consider algorithms which are able to sample from the source arms. 
    We note that studying the problem where we have the ability to sample from both the target and source domains is an interesting problem for future work.
    Our objective is to develop algorithms which will return an $\epsilon$-optimal target arm with high probability. 
    Formally, we focus on an appropriately modified version of the fixed-confidence setting which we define as follows:
    \begin{definition}[$\epsilondelta$-correct]\label{def:epsilon-delta-correct}
        For any $\epsilon \geq 0$ and $\delta \in (0,1)$, we say that an algorithm $\varAlgorithm$ is $\epsilondelta$-correct for the transfer BAI problem if, with probability at least $1 - \delta$, and for every problem instance $(\{\sourceArm_\sourceIndex, \sourceMean_\sourceIndex\}_{\sourceIndex = 1}^{\numSourceArms}, \{\targetArm_\targetIndex, \targetMean_\targetIndex\}_{\targetIndex = 1}^{\numTargetArms}, \varTransferFunction)$, $\varAlgorithm$ stops and returns an $\epsilon$-optimal arm $\varSelectionRule\in [\numTargetArms]$ satisfying $\targetMean_{\varSelectionRule} +\epsilon \geq \max_{\targetIndex \in [\numTargetArms]} \targetMean_\targetIndex$. 
    % For any $\epsilon \geq 0$ and $\delta > 0$, we say that an algorithm $\varAlgorithm$ is $\epsilondelta$-correct for the linear-transfer BAI problem if, for every problem instance $(\{\sourceArm_\sourceIndex, \sourceMean_\sourceIndex\}_{\sourceIndex = 1}^{\numSourceArms}, \{\targetArm_\targetIndex, \targetMean_\targetIndex\}_{\targetIndex = 1}^{\numTargetArms}, \varLinearTransformation)$, with probability at least $1 - \delta$, $\varAlgorithm$ stops and returns an $\epsilon$-optimal arm $\varSelectionRule\in [\numTargetArms]$ satisfying $\targetMean_{\varSelectionRule} > \targetMean_1 - \epsilon$. 
    \end{definition}
    % This notion of correctness is the natural extention of correctness in the traditional fixed confidence setting. 
    % \textcolor{red}{AS: redefine BAI problem in target domain}
    
    As is standard with typical BAI algorithms, an algorithm for the transfer BAI problem is comprised of three components: a sampling rule, a stopping rule, and a selection rule. 
    Letting $\varSigmaAlgebra_\round = \sigma(X_1, \ldots, X_\round)$ denote the $\sigma$-algebra generated by the observations from the source arms up until time $\round$, we have 
        \begin{enumerate}
            \item a sampling rule, $\varSamplingRule_t$, which is a $\varSigmaAlgebra_{t-1}$-measurable function which selects the source arms to pull during round $t$;
            \item a stopping rule, $\varStoppingRule$, which is a $\varSigmaAlgebra_t$-measurable random variable which determines when the algorithm stops;
            \item a selection rule, $\varSelectionRule$, which is a $\varSigmaAlgebra_\varStoppingRule$-measurable function which outputs a guess of the optimal target arm $\targetIndex^*$.
        \end{enumerate}

    \subsection{Assumptions}
    Before proceeding, we briefly discuss our assumptions.
    Our first assumption places restrictions on the class of additive transfer functions which our algorithm is able to handle.
    \begin{assumption}[Assumptions on $\varTransferFunction$]\label{ass:transfer-function}
        We assume that $\varTransferFunction_{\targetIndex, \sourceIndex}$ is continuous at $\sourceMean_\sourceIndex$ for all $(\targetIndex, \sourceIndex) \in [\numTargetArms]\times[\numSourceArms]$.
    \end{assumption}
    
    We additionally assume that the observations from the source MAB instances are sub-Gaussian. 
    \begin{assumption}[$\varSubGaussianParam$-sub-Gaussian Observations]\label{ass:sub-gaussian-observations}
    We assume that the observations from the source arms are $\varSubGaussianParam$-sub-Gaussian so that for any $i \in [\numSourceArms]$ and $\lambda \in \bbR$ the following holds
    \begin{equation}
        \log \bbE_{X\sim\sourceArm_\sourceIndex}[\exp\braket{\lambda (X - \sourceMean_\sourceIndex)}] \leq \frac{\lambda^2 \varSubGaussianParam^2}{2}.
    \end{equation}
    % \aartiComment{What does $X\sim S_i$ mean? Should it be $X\sim P_i$ or $X \sim P_{S_i}$?}
    \end{assumption}
    This assumption is necessary for the concentration inequalities used in the construction of our LUCB-style algorithm given in Section \ref{sec:algorithm}. 
    We note that this, with minimal modification, our assumption, algorithm, and the resulting sample complexity analysis can accommodate arbitrary sub-$\psi$ observations through the use of the concentration inequalities given by Howard et al. \cite{howard2018uniform} --- in Assumption~\ref{ass:sub-gaussian-observations}, we have implicitly set $\psi(\lambda) = \frac{\lambda^2}{2}$.
    However, to simplify the exposition, we limit the scope of this work to sub-Gaussian observations.
    Finally, without loss of generality, we assume that the means are ordered in decreasing order so that $\sourceMean_1 \geq \sourceMean_2 \ldots \geq \sourceMean_\numSourceArms$ and $\targetMean_1  \geq \targetMean_2 \geq \ldots \geq \targetMean_\numTargetArms$.
    We only require the optimal target arm to be unique when $\epsilon = 0$.

%%%%%%%%%%%%%%%%%%%%%%%%%%%%%%%%%%%%%%%%%%%%%%%%%%%%%%%%%%%%%%%%%%%%%%%%%%%%%%%%%%%%%%%%%%%%%%%%%%%%%%%%%%%%%%%%%%%%%%%%%%%%%%%%%%%%%%%%%%%%%%%%%%%%%%%%%%%%%%%%%%%%%%%%%%%%%%%%%%%%    
%%%%%%%%%%%%%%%%%%%%%%%%%%%%%%%%%%%%%%%%%%%%%%%%%%%%%%%%%%%%%%%%%%%%%%%%%%%%%%%%%%%%%%%%%%%%%%%%%%%%%%%%%%%%%%%%%%%%%%%%%%%%%%%%%%%%%%%%%%%%%%%%%%%%%%%%%%%%%%%%%%%%%%%%%%%%%%%%%%%%    
%%%%%%%%%%%%%%%%%%%%%%%%%%%%%%%%%%%%%%%%%%%%%%%%%%%%%%%%%%%%%%%%%%%%%%%%%%%%%%%%%%%%%%%%%%%%%%%%%%%%%%%%%%%%%%%%%%%%%%%%%%%%%%%%%%%%%%%%%%%%%%%%%%%%%%%%%%%%%%%%%%%%%%%%%%%%%%%%%%%%
\section{Related Work}\label{sec:related-work}
    The work most closely resembling ours is a recent line of work on obtaining sample complexity guarantees for Monte Carlo tree search algorithms \cite{garivier2016maximin, kaufmann2017monte, huang2017structured}. 
    % \cite{garivier2016maximin} introduce the maximin action identification problem and show how it can be used to derive sample complexity results in Monte Carlo tree search problems for trees of depth two. 
    % This work was then independently extended to trees of arbitrary depth in \cite{kaufmann2017monte} and \cite{huang2017structured}.
    Specifically, Huang et al. \cite{huang2017structured} approach this problem by first introducing the more general structured BAI problem.
    Their structured BAI framework is the same as our transfer BAI framework, however we choose to use a different name to both emphasize that we are transferring knowledge between multiple MAB instances and to avoid confusing the structured BAI problem with the structured MAB framework described in Lattimore and Munos \cite{lattimore2014bounded} and Gupta et al. \cite{gupta2018exploiting}.
    % \aartiComment{this para gives no info on the relevance to our work except the footnote - can we bring footnote up and compress remaining?}
    
    While Huang et al. \cite{huang2017structured} give a general algorithm for their structured BAI problem, their primary objective was to derive algorithms for the Monte Carlo tree search problem.
    As such, their assumptions consequently make their algorithm inapplicable to wide range of settings including the simple linear setting discussed in Section~\ref{sec:problem_setup}.
    Their Assumption~2(i), which requires the transfer function to be component-wise monotonic, already restricts the applicability of their algorithm to a wide range of problems.
    However, we can resolve this issue by using our confidence sequence construction given in Section~\ref{sec:algorithm}. 
    Their Assumption~2(ii), however, is more troublesome as it requires the confidence sequence of each target arm to be contained in the confidence sequence of at least one source arm.
    To resolve this, Huang et al. \cite{huang2017structured} briefly mention a weaker assumption wherein the confidence sequence of each target arm must be contained in a scaled and shifted version of a source arm's confidence sequence --- however, this weaker assumption is still inapplicable even in the linear setting.
    Additionally, as we show in Appendix~\ref{appendix:micro_lucb}, the resulting sample complexity for this modified algorithm is significantly worse than the sample complexity of our algorithm.
    Finally, we note that the assumptions we make are incomparable to the assumptions made in Huang et al. \cite{huang2017structured} as neither is more or less general than the other.
    
    The simpler linear setting subsumed by our framework, where the transfer function takes the form $\varTransferFunction(\sourceMean) = \varLinearTransformation \sourceMean$ also coincides with the Transductive Linear Bandit problem studied in Fiez et al. \cite{fiez2019sequential} and Katz et al. \cite{katz2020empirical} when the sampling vectors are the standard basis of $\mathbb R^{\numSourceArms}$.
    However, it is not clear how to extend the ideas presented in these works to the additive setting since the algorithms strongly utilize the linearity in the problem.
    
    The `partition identification' problem  introduced by Juneja and Krishnasamy \cite{juneja2018sample} is also related to our work.  
    In fact, their framework can be seen as a generalization of the problem studied here. 
    However, in their work, Juneja and Krishnasamy \cite{juneja2018sample} primarily focus on providing lower bounds for variations of the partition identification problem and only briefly discuss an asymptotically optimal algorithm towards the end of their work. 
    Additionally, it is known that Confidence-Interval style algorithms (like the one we propose) outperform their Track-And-Stop style algorithm in so-called moderate-confidence regimes\footnote{By moderate confidence regimes we mean regimes where $\delta$ is moderately small, i.e when $\delta \approx .05$ or when it is inverse-polynomial in the number of measurements \cite{simchowitz2017simulator}.} \cite{simchowitz2017simulator}.
    Moreover, it is not clear that the algorithm they provide is can even implementable in the linear setting because implementing it requires solving a constrained optimization problem over a (possibly) non-convex set.
    Finally, the analysis in \cite{juneja2018sample} only provides asymptotic guarantees for their algorithm while we provide explicit finite-time guarantees for our algorithm. 
    % \todo{I feel like I should discuss computational issues somewhere in the paper, but I am not sure where. The two main computational issues are in computing the min/max of $\varTransferFunction$ over a closed interval and computing $B_\round, C_\round$ as it can be the case that this cannot be implemented efficiently in some settings where there are combinatorially many target arms. \tb{If so, where should I add this?}}

    \subsection{Subsumed Settings}\label{subsec:special-settings}
        Finally, as we alluded in Section~\ref{sec:problem_setup}, we now describe how the additive-transfer framework studied here subsumes a range of existing pure exploration problems. 
        In Section~\ref{sec:results}, we instantiate our sample complexity results for some of the problems mentioned below.
        
        \tb{TopK Identification.} In the TopK problem \cite{kalyanakrishnan2012pac, kaufmann2013information}, the objective is to identify the $K$ arms with the largest means. 
        To recover this problem in our formulation, we define the target means as follows. 
        We define a target arm $\targetArm_\armSet$ for each set $\armSet \in 2^{[\numTargetArms]}$ satisfying $|\armSet| = K$.
        The mean of this target arm is then defined as $\targetMean_\armSet = \sum_{\sourceIndex \in \armSet} \sourceMean_\sourceIndex$.
        
        \tb{Thresholding Bandits.} In the Thresholding Bandits problem \cite{locatelli2016optimal}, the objective is to identify the set of arms whose means are greater than some fixed threshold $\threshold \in \bbR$. 
        This problem is subsumed by the property testing problem mentioned earlier. To see this, we simply set, for each $\sourceIndex \in [\numSourceArms]$, $\propertySet_\sourceIndex = (\threshold, \infty)$. Then for every set $\armSet \in 2^{[n]}$ define the mean of target arm $\targetArm_\armSet$ as $\targetMean_\armSet = \sum_{\sourceIndex \in \armSet} \mathbb I_{\mathcal C_i}(\sourceMean_\sourceIndex)$.
        
        \tb{Combinatorial Pure Exploration.} As a final example, we show how our framework generalizes the Combinatorial Pure Exploration problem proposed by Chen et al. \cite{chen2014combinatorial}. This problem is defined by a decision class $\decisionClass \subseteq 2^{[\numSourceArms]}$ and the objective is to identify an element $\armSet \in \decisionClass$ satisfying $\armSet \in \argmax{\armSet \in \decisionClass} \sum_{\sourceIndex \in \armSet} \sourceMean_\sourceIndex$. It is easy to see that this problem fits into our framework by defining a target mean $\targetMean_\armSet = \sum_{\sourceIndex \in \armSet} \sourceMean_\sourceIndex$.  The Combinatorial Pure Exploration problem additionally subsumes a number of additional problems previously studied in the literature, including the examples discussed above. For more examples of subsumed problems and additional discussions, we refer the reader to the literature on this problem \cite{cao2019disagreement, chen2014combinatorial, chen2016pure, chen2017nearly, gabillon2016improved}.

%%%%%%%%%%%%%%%%%%%%%%%%%%%%%%%%%%%%%%%%%%%%%%%%%%%%%%%%%%%%%%%%%%%%%%%%%%%%%%%%%%%%%%%%%%%%%%%%%%%%%%%%%%%%%%%%%%%%%%%%%%%%%%%%%%%%%%%%%%%%%%%%%%%%%%%%%%%%%%%%%%%%%%%%%%%%%%%%%%%%    
%%%%%%%%%%%%%%%%%%%%%%%%%%%%%%%%%%%%%%%%%%%%%%%%%%%%%%%%%%%%%%%%%%%%%%%%%%%%%%%%%%%%%%%%%%%%%%%%%%%%%%%%%%%%%%%%%%%%%%%%%%%%%%%%%%%%%%%%%%%%%%%%%%%%%%%%%%%%%%%%%%%%%%%%%%%%%%%%%%%%    
%%%%%%%%%%%%%%%%%%%%%%%%%%%%%%%%%%%%%%%%%%%%%%%%%%%%%%%%%%%%%%%%%%%%%%%%%%%%%%%%%%%%%%%%%%%%%%%%%%%%%%%%%%%%%%%%%%%%%%%%%%%%%%%%%%%%%%%%%%%%%%%%%%%%%%%%%%%%%%%%%%%%%%%%%%%%%%%%%%%%
\section{Algorithm}\label{sec:algorithm}
    In this section, we present the Transfer LUCB (T-LUCB) algorithm, a variant of the LUCB algorithm        \cite{kalyanakrishnan2012pac} used in the fixed-confidence BAI setting. 
    Like the LUCB algorithm, our T-LUCB algorithm is based on constructing confidence sequences which are time-uniform confidence intervals on the sample means. 
    Before presenting the \TLUCB algorithm, we first discuss the construction of our confidence sequences. 
    
    % \ojComment{Talk about $N_i(t)$ somewhere here}
    To construct the confidence sequences on the source arms we use standard Hoeffding-like confidence sequences \cite{kaufmann2016complexity, howard2018uniform} and define the Lower Confidence Bound (LCB), Upper Confidence Bound (UCB), and Confidence Interval (CI) sequences as follows. 
    Recall that {$I_s$ denotes the arm that is pulled at time $s$}.
    We let $N_{\sourceIndex}(t) = \sum_{s = 1}^{t-1} \indicator{I_s = i}$ denote the number of times that source arm $\sourceIndex$ has been pulled at the start of round $t$.
    Additionally, we let {$\widehat \mu_t(i) = \frac{1}{N_i(t)}\sum_{s = 1}^{t - 1} X_s \indicator{I_s = i}$} denote the empirical mean of arm $i$ at the beginning of round $\round$.
    Then, at $t=0$, we set the lower and upper confidence bounds for source arm $i$ as $\LCB{\sourceArm}{0}{\sourceIndex}{\delta} = -\infty$, $\UCB{\sourceArm}{0}{\sourceIndex}{\delta} = +\infty$. 
    Next, for $t \geq 1$, we recursively define the confidence sequences as:
        \edit{\begin{multline}
            \UCB{\sourceArm}{t}{\sourceIndex}{\delta}  \coloneqq \min\bigg\{\UCB{\sourceArm}{t-1}{\sourceIndex}{\delta}, \\ \widehat{\sourceMean}_t(\sourceIndex) + \beta(N_\sourceIndex(t),\edit{\delta/(2\numSourceArms)})\bigg\}\label{eq:source_lcb},
        \end{multline}
        \vspace{-0.5cm}
        \begin{multline}
            \LCB{\sourceArm}{t}{\sourceIndex}{\delta} \coloneqq \max\bigg\{\LCB{\sourceArm}{t-1}{\sourceIndex}{\delta}, \\ \widehat{\sourceMean}_t(\sourceIndex) - \beta(N_\sourceIndex(t),\edit{\delta/(2\numSourceArms)})\bigg\}\label{eq:source_ucb},
        \end{multline}}
        \begin{equation}
            \CI{\sourceArm}{t}{\sourceIndex}{\delta} \coloneqq [\LCB{\sourceArm}{t}{\sourceIndex}{\delta}, \UCB{S}{t}{\sourceIndex}{\delta}].
        \end{equation}
            
        % \begin{align}
        %         \UCB{\sourceArm}{t}{\sourceIndex}{\delta}   &\coloneqq \min\braket{\UCB{\sourceArm}{t-1}{\sourceIndex}{\delta}, \widehat{\sourceMean}_t(\sourceIndex) + \sqrt{\frac{\beta(N_\sourceIndex(t),\edit{\delta/(2\numSourceArms)})}{N_\sourceIndex(t)}}}\label{eq:source_lcb}, \\
        %         %%%%
        %         \LCB{\sourceArm}{t}{\sourceIndex}{\delta} &\coloneqq \max\braket{\LCB{\sourceArm}{t-1}{\sourceIndex}{\delta}, \widehat{\sourceMean}_t(\sourceIndex) - \sqrt{\frac{\beta(N_\sourceIndex(t),\edit{\delta/(2\numSourceArms)})}{N_\sourceIndex(t)}}}\label{eq:source_ucb}, \\
        %         %%%%
        %         \CI{\sourceArm}{t}{\sourceIndex}{\delta}    &\coloneqq [\LCB{\sourceArm}{t}{\sourceIndex}{\delta}, \UCB{S}{t}{\sourceIndex}{\delta}].
        % \end{align}
        Here $\beta(\cdot, \cdot)$ is a function which controls the rate at which the confidence intervals shrink.
        As an example, $\beta$ can be taken to be the so-called ``polynomial stitched boundary'' \cite[Eq.(6)]{howard2018uniform}:
        \begin{equation}\label{eq:finite-lil}
            \edit{\beta(t, \delta) \coloneqq 1.7\sqrt{\frac{\sigma^2\log\log\parenthesis{ 2t\sigma^2} + 0.72\log\frac{5.2}{\delta}}{t}}.}
        \end{equation}
        More generally, for the results given in Section~\ref{sec:results} to hold, $\beta$ must satisfy the following condition:
        \begin{equation}\label{eq:beta-condition}
            \bbP\braket{\exists t \geq 1: \mu_\sourceIndex \not \in \CI{\sourceArm}{t}{\sourceIndex}{\delta}} \leq \delta.
        \end{equation}
        The choice of $\beta$ in \eqref{eq:finite-lil} satisfies the above condition.
        Next, we use the source arm confidence sequences to construct confidence sequences on the target arms as follows:
        \edit{\begin{align}
            \LCB{\targetArm}{\round}{\targetIndex}{\delta}   &\coloneqq \sum_{\sourceIndex = 1}^{\numSourceArms} \min_{m_\sourceIndex \in \CI{\sourceArm}{\round}{\sourceIndex}{\delta}}\varTransferFunction_{\targetIndex, \sourceIndex}\parenthesis{m_\sourceIndex},\\
            \UCB{\targetArm}{\round}{\targetIndex}{\delta}   &\coloneqq \sum_{\sourceIndex = 1}^{\numSourceArms} \max_{m_\sourceIndex \in \CI{\sourceArm}{\round}{\sourceIndex}{\delta}}\varTransferFunction_{\targetIndex, \sourceIndex}\parenthesis{m_\sourceIndex},\\
            \CI{\targetArm}{\round}{\targetIndex}{\delta}    &\coloneqq [\LCB{\targetArm}{\round}{\targetIndex}{\delta}, \UCB{\targetArm}{\round}{\targetIndex}{\delta}].
        \end{align}}

        The intuition for the above construction is as follows. By constructing the source confidence sequences as defined in equations~\eqref{eq:source_lcb} and~\eqref{eq:source_ucb}, and choosing $\beta$ to satisfy condition~\eqref{eq:beta-condition}, we can control the deviations of the source samples means from the true source means. This in turn implies that the constructed target confidence sequences are well-behaved in the sense that they will contain the target arm means with high probability.
        This intuition is formalized by Lemma~\ref{lem:good-event-prob} in the Appendix.
        % \todo{Add lemma to appendix}
        \paragraph{The T-LUCB Algorithm} 
        We are now ready to introduce the \TLUCB algorithm which is stated in Algorithm~\ref{alg:transfer-lucb}. 
        During each round, the algorithm selects two target arms $B_t$ and $C_t$ with the objective of separating the LCB of $B_t$ from the UCB of $C_t$.
        After selecting $B_t$ and $C_t$, the algorithm samples the source arms $I_t$ and $J_t$ which respectively have the largest contributions to the length of the confidence sequences of $B_t$ and $C_t$. 
        Formally, we define the following quantity
        \begin{equation}\label{eq:tlucb-length}
               \ealgLength{\sourceIndex}{\targetIndex}{t} = \max_{m \in \CI{\sourceArm}{\round}{\sourceIndex}{\delta}} \varTransferFunction_{\targetIndex, \sourceIndex}(m) -  \min_{m \in \CI{\sourceArm}{\round}{\sourceIndex}{\delta}} \varTransferFunction_{\targetIndex, \sourceIndex}(m),
        \end{equation}
        which quantifies the amount of uncertainty that source arm $\sourceIndex$ contributes to target arm $\targetIndex$.
        The algorithm stops when the LCB of $B_t$ is greater than the UCB of $C_t$.
        Finally the algorithm selects $B_t$ as its guess for the optimal target arm.

        \begin{algorithm}[!ht]~\caption{Additive Transfer LUCB}
        \label{alg:transfer-lucb}
            \textbf{Input} $\delta > 0, \epsilon \geq 0$, $\varTransferFunction$, $\sigma^2$\;
            Sample each source arm once\;
            \For{$\round = 1, 2, \ldots$}{
                % \ojComment{Macro this}
                $B_\round = \argmax{\targetIndex \in [\numTargetArms]} \LCB{\targetArm}{\round}{\targetIndex}{\delta}$\;
                $C_\round = \argmax{\targetIndex \in [\numTargetArms], \targetIndex \neq B_\round} \UCB{\targetArm}{\round}{\targetIndex}{\delta}$\;
                \If{$\LCB{\targetArm}{\round}{B_\round}{\delta} + \epsilon \geq \UCB{\targetArm}{\round}{C_\round}{\delta}$}{
                    \Return $\varSelectionRule = B_t$\;
                }
                $I_t = \argmax{\sourceIndex \in [\numSourceArms]} \ealgLength{\sourceIndex}{B_\round}{\round}$ \;
                $J_t = \argmax{\sourceIndex \in [\numSourceArms]} \ealgLength{\sourceIndex}{C_\round}{\round}$ \;
                Observe $X_{t, 1} \sim \sourceArm_{I_t}$ and $X_{t, 2} \sim \sourceArm_{J_t}$\;
                % \State Update $[\LCB{\sourceArm}{t}{I_t}{\delta}, \UCB{\sourceArm}{t}{I_t}{\delta}]$ and $[\LCB{\sourceArm}{t}{J_t}{\delta}, \UCB{\sourceArm}{t}{J_t}{\delta}]$
            }
            % \caption{The T-LUCB algorithm. Lines 3-6 describe our sampling rule, line 8 describes our stopping rule, and line 9 describes our selection rule. \ojComment{Maybe move this into the actual algorithm as comments? Just remove it, not a complex algorithm.}}
            
    \end{algorithm}
    
%%%%%%%%%%%%%%%%%%%%%%%%%%%%%%%%%%%%%%%%%%%%%%%%%%%%%%%%%%%%%%%%%%%%%%%%%%%%%%%%%%%%%%%%%%%%%%%%%%%%%%%%%%%%%%%%%%%%%%%%%%%%%%%%%%%%%%%%%%%%%%%%%%%%%%%%%%%%%%%%%%%%%%%%%%%%%%%%%%%%    
%%%%%%%%%%%%%%%%%%%%%%%%%%%%%%%%%%%%%%%%%%%%%%%%%%%%%%%%%%%%%%%%%%%%%%%%%%%%%%%%%%%%%%%%%%%%%%%%%%%%%%%%%%%%%%%%%%%%%%%%%%%%%%%%%%%%%%%%%%%%%%%%%%%%%%%%%%%%%%%%%%%%%%%%%%%%%%%%%%%%    
%%%%%%%%%%%%%%%%%%%%%%%%%%%%%%%%%%%%%%%%%%%%%%%%%%%%%%%%%%%%%%%%%%%%%%%%%%%%%%%%%%%%%%%%%%%%%%%%%%%%%%%%%%%%%%%%%%%%%%%%%%%%%%%%%%%%%%%%%%%%%%%%%%%%%%%%%%%%%%%%%%%%%%%%%%%%%%%%%%%%
\section{Results}\label{sec:results}
   In this section, we analyze the \TLUCB algorithm presented in Section~\ref{sec:algorithm}.
        Our first result shows that, regardless of the sampling rule, the stopping rule and selection rule of Algorithm~\ref{alg:transfer-lucb} are sufficient to give an $\epsilondelta$-correct algorithm. The proof of this result can be found in Section~\ref{sec:proofs} in the Appendix. 
    \begin{theorem}\label{thm:correctness}
        Suppose that $\beta$ satisfies condition~\eqref{eq:beta-condition}. Then, any algorithm which stops when there exists an arm $\targetIndex \in [\numTargetArms]$ such that
        \begin{equation}
            \LCB{\targetArm}{t}{\targetIndex}{\delta} \edit{+ \epsilon}\geq \UCB{\targetArm}{t}{\targetIndex'}{\delta},
        \end{equation}
        for all \edit{$\targetIndex' \neq \targetIndex$}, and selects the arm $\varSelectionRule = \targetIndex$, will with probability at least $1 - \delta$, choose an arm satisfying $\targetMean_{\varSelectionRule} \geq \targetMean_1 - \epsilon$.
    \end{theorem}

    We now shift our attention towards providing a high probability upper bound on the sample complexity of Algorithm~\ref{alg:transfer-lucb}. To present our function specific upper-bound on the sample complexity we first introduce some additional notation. 
    We remark that due to the generality of our framework our generic sample complexity bound is presented implicitly, and is difficult to immediately interpret. 
    As such, we will present explicit bounds for some instantiations of our problem in the following subsection.
    
    First, we define 
    \begin{equation}
        \varSparsity_\targetIndex \coloneqq |\{\sourceIndex : \varTransferFunction_{\targetIndex, \sourceIndex}(x) \neq \varTransferFunction_{\targetIndex, \sourceIndex}(y), \ \forall x, y \in \bbR\}|,
    \end{equation}
    which measures the number of source arms which contribute to the uncertainty of a target arm. 
    For the property testing problem, $\varSparsity_\targetIndex = |M|$, which is the number of terms in the sum $\sum_{\sourceIndex \in \armSet} \mathbb I_{\propertySet_\sourceIndex}(\sourceMean_\sourceIndex)$.
    For linear transfer functions, $\varSparsity_\targetIndex = |\{\sourceIndex: \varLinearTransformation_{\targetIndex, \sourceIndex} \neq 0\}$ which measures the sparsity of the vector $\varLinearTransformation_\targetIndex$.

    Next, with a slight abuse of notation, we define the following quantity which has a similar form to equation~\ref{eq:tlucb-length}
    \begin{multline}
                \ecomplexityLength{\sourceIndex}{\targetIndex}{\round}{x} = \max_{m \in [x, x + 2\beta(t, \delta)]} \varTransferFunction_{\targetIndex, \sourceIndex}(m) \\- \min_{m \in [x, x + 2\beta(\round, \delta)]} \varTransferFunction_{\targetIndex, \sourceIndex}(m).
    \end{multline}

    This term quantifies how much source arm $\sourceIndex$ contributes to the confidence interval of target arm $\targetIndex$ when the LCB of source arm $\sourceIndex$ is $x$. 
    For the property testing problem, we have 
    % \textcolor{red}{define $C_i^c$ or just replace with $\not\in C_i$ here and below}
    \begin{equation}\label{eq:property-testing-length}
        \ecomplexityLength{\sourceIndex}{\armSet}{\round}{x} = 
        \begin{cases}
            0       &   \text{if } [x, x + 2\beta(t, \delta)] \subseteq \propertySet_\sourceIndex, \sourceIndex \in \armSet  \\
            0       &   \text{if } [x, x + 2\beta(t, \delta)] \subseteq \propertySet_\sourceIndex^c, \sourceIndex \in \armSet \\
            0       &   \text{if } \sourceIndex \not \in \armSet \\
            \infty &   \text{otherwise}
        \end{cases},    
    \end{equation}
    where $\propertySet^c_{\sourceIndex}$ is the complement $\propertySet_\sourceIndex$ and we have taken the convention that $\infty - \infty = 0$.
    For linear transfer functions, this quantity is independent of $x$ so that $\ecomplexityLength{\sourceIndex}{\targetIndex}{\round}{x} = 2|\varLinearTransformation_{\targetIndex, \sourceIndex}|\beta(\round, \delta)$. 
    Having defined this quantity, we are now ready to define an upper bound on the number of times source arm $\sourceIndex$ needs to be sampled in order to determine if target arm $\targetIndex$ is $\epsilon$-optimal. First, we set
    \begin{multline}
         \varStoppingTime_{\targetIndex, \sourceIndex} = \min\bigg\{t \in \bbN: \sup_{x \in [\sourceMean_\sourceIndex - 2\beta(\round, \delta), \sourceMean_\sourceIndex]}  \ecomplexityLength{\sourceIndex}{\targetIndex}{\round}{x} \\< \frac{\max\braket{|\varImaginaryTargetArm - \targetMean_\targetIndex|, \epsilon/2}}{\varSparsity_\targetIndex}\bigg\},
    \end{multline}
    where $\varImaginaryTargetArm \coloneqq \frac{\targetMean_1 + \targetMean_2}{2}$. 
    Then, we define 
    \begin{equation}\label{eq:source-arm-stopping-time}
        \varStoppingTime_{\sourceIndex} = \max_{\targetIndex \in [\numTargetArms]} \varStoppingTime_{\targetIndex, \sourceIndex},
    \end{equation}
    which represents the number of times source arm $\sourceIndex$ must be pulled in order to determine which target arms are $\epsilon$-optimal. We are now ready to state our sample complexity result.
    \begin{theorem}[Sample Complexity Upper Bound of Algorithm~\ref{alg:transfer-lucb}]\label{thm:tlucb-sample-complexity}
    Let $\varStoppingTime$ denote the stopping time of Algorithm~\ref{alg:transfer-lucb}. Then with probability at least $1 - \delta$, we have that 
    \begin{equation}
        \varStoppingTime \leq \sum_{\sourceIndex \in [\numSourceArms]} \varStoppingTime_{\sourceIndex}.
    \end{equation}
    \end{theorem}
    Note that this sample complexity bound is independent of the number of target arms. This fact allows us to recover the sample complexity of some existing problems as we show in the following subsection.
    
    % For the property testing problem, define the following quantity,
    % \textcolor{red}{C should have a subscript}
    
    Theorem~\ref{thm:tlucb-sample-complexity} implies the following sample complexity result for the property testing problem. 
    \begin{corollary}\label{cor:property-testing}
     Let $\varStoppingRule$ denote the stopping time of Algorithm \ref{alg:transfer-lucb} for the property testing problem and define 
        \begin{equation}\label{eq:sample-complexity-term}
            H \coloneqq \uosum{\sourceIndex = 1}{\numSourceArms} \frac{2}{\Delta^2_{\propertySet_\sourceIndex}(\sourceMean_\sourceIndex)},
        \end{equation}
        where
        \begin{equation*}
        \Delta_{\propertySet_\sourceIndex}(\sourceMean_\sourceIndex) = 
        \begin{cases}
            \inf_{x \in \propertySet^c_\sourceIndex} |x - \sourceMean_\sourceIndex| & \text{if } \sourceMean_\sourceIndex \in \propertySet_\sourceIndex \\
            \inf_{x \in \propertySet_\sourceIndex} |x - \sourceMean_\sourceIndex| & \text{if } \sourceMean_\sourceIndex \not \in \propertySet_\sourceIndex \\
        \end{cases}.
        \end{equation*}
        Then\footnote{We use $\Tilde{O}$ to refer to sample complexity results which are correct up to constant and $\log\log$ factors.}  with probability at least $1 - \delta$,
        \begin{equation}
            \tau \leq \Tilde{O}\left(H \log\left(\frac{1}{\delta}\right)\right).
        \end{equation}
    \end{corollary}
    
    For linear transfer functions, we obtain the following result.
    \begin{corollary}\label{eq:cor-linear-sample-complexity}
     Let $\varStoppingRule$ denote the stopping time of Algorithm \ref{alg:transfer-lucb} for the linear transfer setting and define 
        \begin{equation}\label{eq:sample-complexity-term}
            \sampleComplexity{\varLinearTransformation, \targetMean, \sourceMean}{\epsilon} \coloneqq \uosum{\sourceIndex = 1}{\numSourceArms} \max_{\targetIndex \in [\numTargetArms]} \bigg\{\frac{\varSparsity_\targetIndex^2|\varLinearTransformation_{\targetIndex, \sourceIndex}|^2}{\varMaxTerm{\targetIndex}^2}\bigg\}.
        \end{equation}
        Then with probability at least $1 - \delta$,
        % \footnote{We use $\Tilde{O}$ to refer to sample complexity results which are correct up to constant and $\log\log$ factors.} 
        % \todo{Need to move this footnote, but not sure to where.}
        \begin{equation}
            \tau \leq \Tilde{O}\left(\sampleComplexity{\varLinearTransformation, \targetMean, \sourceMean}{\epsilon}\log\left(\frac{1}{\delta}\right)\right).
        \end{equation}
    \end{corollary}
    \subsection{Instantiations of Theorem~\ref{thm:tlucb-sample-complexity}}\label{sec:instantiations}
        We now proceed to instantiate the sample complexity bound of Theorem~\ref{thm:tlucb-sample-complexity} for some previously studied settings. 
        In each of these settings we state an explicit bound which is a direct corollary of the sample complexity bound from Theorem~\ref{thm:tlucb-sample-complexity}. 
        Proofs of these results can be found in Appendix~\ref{app:instantiations}.

        \tb{BAI.} To recover the Best Arm Identification problem, we simply set $\targetMean_\targetIndex = \sourceMean_\sourceIndex$ so that the mean of each target arm is simply the mean of one of the source arms. First, we set $\overline{\sourceMean} = \frac{\sourceMean_1 + \sourceMean_2}{2}$.
        Then Theorem~\ref{thm:tlucb-sample-complexity} implies that \[\varStoppingTime \leq \Tilde{O}\parenthesis{\sum_{\sourceIndex = 1}^{\numSourceArms} \frac{1}{(\overline{\sourceMean} - \sourceMean_i)^2} \log (1/\delta)}.\] This recovers the  sample complexity of the original LUCB algorithm \cite{kalyanakrishnan2012pac}.
        
        \tb{Thresholding Bandits.} Here, Theorem~\ref{thm:tlucb-sample-complexity} implies that \[\varStoppingTime \leq \Tilde{O}\parenthesis{\sum_{\sourceIndex \in [\numSourceArms]} \frac{1}{(\sourceMean_\sourceIndex - \threshold)^2}\log(1/\delta)},\] which matches, up to iterated logarithmic factors, the problem's sample complexity lower bound given for the fixed confidence setting \cite{locatelli2016optimal}.
        
        % \textcolor{red}{Are we not including some simple combinatorial pure exploration example? Also, we can still include linear  transfer as an example - I just wanted to add a nonlinear example, not necessarily remove linear}
        
        \tb{TopK.} One example of a Combinatorial Pure Exploration problem is the so-called TopK problem where we wish to identify the $K$ largest means our of $\numSourceArms$ arms. This problem can be recovered in the CPE framework by letting $\decisionClass$ to be the all subsets of $\{1, \ldots, \numSourceArms\}$ with cardinality $K$. To state our sample complexity results in this setup, we first define $\overline{\sourceMean} = \frac{\sourceMean_{K} + \sourceMean_{K + 1}}{2}$.
        Then, Theorem~\ref{thm:tlucb-sample-complexity} implies that \[\varStoppingTime \leq \Tilde{O}\parenthesis{\sum_{\sourceIndex \in [\numSourceArms]} \frac{K^2}{(\sourceMean_\sourceIndex - \bar{\sourceMean})^2}\log(1/\delta)}.\] We remark that this sample complexity result is suboptimal by a factor of $K^2$ \cite{kalyanakrishnan2012pac, kaufmann2013information}. However, we conjecture that this is the price of generality of our framework. We refer the reader to Section~\ref{sec:conclusion} for more discussion on this.
        
        % \tb{Combinatorial Pure Exploration} {\color{red} Here, Theorem~\ref{thm:tlucb-sample-complexity} implies that $\varStoppingTime \leq \Tilde{O}\parenthesis{\sum_{\sourceIndex \in [\numSourceArms]} \frac{1}{(\sourceMean_\sourceIndex - \threshold)^2}\log(1/\delta)}$ which matches, up to iterated logarithmic factors, the problem's sample complexity lower bound given for the fixed confidence setting \cite{locatelli2016optimal}.}
        % \todo[color=black]{Need to read through the proofs in \cite{chen2014combinatorial} and verify that we obtain their results.}}

\section{Conclusion}\label{sec:conclusion}
In this work we presented and analyzed an algorithm for leveraging additive relationships between two MAB instances to identify the best arm in a MAB instance without ever sampling from it. 

A first direction for future work would be to investigate if an algorithm for the additive transfer setting can recover the correct sample complexity results for the specialized settings such as the TopK problem.
We conjecture that this is not possible.
This is because algorithms for these simpler settings either implicitly or explicitly utilize a type of well-ordering property of the problem which does not generally hold for non-linear additive transfer functions.
This well-ordering property is made explicit in the work of Gabillon et al. \cite{gabillon2016improved}, and is implicitly utilized in the work of Fiez et al. \cite{fiez2019sequential}.
We briefly illustrate this well-ordering property.
Consider two target arms where we have a sub-optimal target-arm $\targetIndex$, and compare it with the target arm $\complement{\targetIndex}$ which determines $\targetIndex$ is sub-optimal with the fewest number of samples.
The well-ordering property in the linear setting states that the number of samples required to determine that $\targetIndex$ is sub-optimal is always fewer than than the number of samples requires to determine that $\complement{\targetIndex}$ is sub-optimal (assuming that $\complement{\targetIndex}$ is not optimal)\footnote{See Proposition 4 in the Appendix of \cite{gabillon2016improved}}.
It is possible to construct non-linear additive transfer functions for which this property does not hold, and as such, it is not clear if any algorithm can adapt to this well-ordering property when it is satisfied.

An issue with our proposed framework is that we assume the transfer function is known in advance.
Another interesting direction of future research is to study how to alleviate this requirement so that, for example, the transfer function can be learned from historical data. 
If this approach is taken, it may no longer be possible to identify a $\epsilon$-optimal target arm as the error introduced from estimating the transfer function might lead to a scenario where the true optimal target arm is not the optimal target arm under the approximate transfer function. 
% \textcolor{red}{AS: I think all you mean is BAI may not be possible? not clear - lets discuss when we meet} 
We believe in this setting a more reasonable criterion to study is the \emph{simple regret} \cite{bubeck2009pure} under the assumption that the learned transfer function is close in norm to the true transfer function.

Furthermore, in this work we consider the setting where  we are unable to sample from the target MAB instance. 
Another interesting direction would be in developing algorithms which are able to sample from the target MAB instance with the caveat that doing so has some additional cost. 
This type of setting seems natural as it is often the case that making direct measurements of some system can be significantly more expensive than taking noisier auxiliary measurements of the system.
A concrete example of this is in the sim-to-real problem, where collecting observations from the real world is significantly more expensive than collecting observations from a computer simulation. 
Additionally, the ability to sample the target arm can allow for learning or refining the transfer function on the fly using few transfer queries.


% \section*{Acknowledgment}
% \textcolor{red}{Do we add acknowledgements now or later?}

\bibliographystyle{unsrt}
\bibliography{ref}

\newpage
\onecolumn
\appendices

\section{Comparison with Micro-LUCB from \cite{huang2017structured}}\label{appendix:micro_lucb}
In this section we provide an in-depth discussion and comparison of our Algorithm~\ref{alg:transfer-lucb} and a variant of the Micro-LUCB algorithm which is suitable for linear transfer functions.
We first restate their assumptions and demonstrate why the do not hold for our setting. In this assumption, we note that $\leq$ denotes a component wise ordering so $u \leq v$ is equivalent to stating $u_i \leq v_i$ for all $i$.
    \begin{assumption}[Assumption 2 of \cite{huang2017structured}]\label{ass:micro-lucb-assumption}
        The following hold:
        \begin{enumerate}[(i)]
            \item The mapping function $f$ is monotonous with respect to the partial order of vectors: for any $u, v \in \bbR^{\numSourceArms}$, $u\leq v$ implies $f(u)\leq f(v)$.\label{ass:micro-lucb-assumption_i}
            \item For any $u, v \in \bbR^{\numSourceArms}$, $u \leq v$, $\targetIndex \in [\numTargetArms]$, the set $D(\targetIndex, u, v) \coloneqq \{\sourceIndex \in [\numSourceArms] : [f_\targetIndex(u), f_\targetIndex(v)] \subset [u_\sourceIndex, v_\sourceIndex]\}$ is non-empty\label{ass:micro-lucb-assumption_ii}.
        \end{enumerate}
    \end{assumption}

% \ojComment{Should I use $f_j$ or $\varLinearTransformation_j$?}
To see that Assumption~\ref{ass:micro-lucb-assumption} (i) is not satisfied for arbitrary linear transformations, we set some entries of the associated matrix to be negative, then there will exist some $\targetIndex$ for which $f_\targetIndex$ is not monotonous. 
This assumption is used to define the confidence intervals on the target arms, and without it, their proof of correctness does not hold.
We modify the assumption to the following which trivially holds true for any function:
    \begin{assumption}
        The mapping function $f$ is monotone with respect to the partial order of vectors: for any $u, v \in \bbR^{\numSourceArms}$, $u\leq v$ implies $\min_{u \leq m \leq v} f(m)\leq \max_{u \leq m \leq v}f(m)$\label{ass:modified_micro-lucb-assumption_i}.
    \end{assumption}
    It can be verified that if our target confidence sequences are constructed as 
    \begin{align}\label{eq:target_confidence_sequences}
                \LCB{\targetArm}{t}{\targetIndex}{\delta}   &\coloneqq \underset{m_i \in \CI{\sourceArm}{t}{\sourceIndex}{\delta}}{\min} f_{\targetIndex}(m), \\
                \UCB{\targetArm}{t}{\targetIndex}{\delta}   &\coloneqq \underset{m_i \in \CI{\sourceArm}{t}{\sourceIndex}{\delta}}{\max} f_{\targetIndex}(m), \\
                \CI{\targetArm}{t}{\targetIndex}{\delta}    &\coloneqq [\LCB{\sourceArm}{t}{\sourceIndex}{\delta}, \UCB{\sourceArm}{t}{\sourceIndex}{\delta}],
    \end{align}
    then the T-LUCB stopping rule and selection rule can be applied to any algorithm to give an $\epsilondelta$-correct algorithm.
    The proof of this is a simple modification of the proof of Theorem~\ref{thm:correctness} where we simply replace the construction of the target confidence sequences given in Section~\ref{sec:algorithm} with the construction defined above. %\aartiComment{remove last sentence as you use this construction in Section 4 now}.

We now switch our attention to Assumption~\ref{ass:micro-lucb-assumption} (ii). 
In short, Assumption~\ref{ass:micro-lucb-assumption} (ii) requires that for each target arm confidence interval, there exists at least one source arm confidence interval which contains the target arm confidence interval. 
This assumption is used to determine the set of source arms which should be sampled in the Micro-LUCB algorithm. 
Indeed, it is integral for the algorithm since, if the assumption is not satisfied, the sampling rule is not well defined. 
While this assumption is not directly satisfied for the linear setting, \cite{huang2017structured} mention one avenue for weakening the assumption so that it is satisfied for a larger class of functions. This weaker assumption is as follows:

        There exists some $a > 0, b \in \bbR$ such that for any $u, v \in \bbR^{\numSourceArms}$, $u \leq v$, $\targetIndex \in [\numTargetArms]$, the set $\Tilde{D}(\targetIndex,u,v) = \{\sourceIndex \in [\numSourceArms] : [f_\targetIndex(u), f_\targetIndex(v)] \subset [au_\sourceIndex + b, av_\sourceIndex + b]\}$ is non-empty.    
    
However, this assumption also is not well defined as $[f_a(u), f_a(v)]$ is not an interval unless $f_a$ is component-wise monotonically increasing.
To fix this, we propose the following assumption:
    
    \begin{assumption}[Modified Assumption 2(ii) of \cite{huang2017structured}]\label{ass:modified-micro-lucb-assumption}
        There exists some $a_\sourceIndex > 0, b_\sourceIndex \in \bbR$ such that for any $u, v \in \bbR^{\numSourceArms}$, $u \leq v$, $a \in [\numTargetArms]$, the set $\Tilde{D}(\targetIndex,u,v) = \{\sourceIndex \in [\numSourceArms] : \left[\min_{u \leq m \leq v} f_{\targetIndex}(m), \max_{u \leq m \leq v} f_{\targetIndex}(m)\right] \subset [a_\sourceIndex u_\sourceIndex + b_\sourceIndex, a_\sourceIndex v_\sourceIndex + b_\sourceIndex]\}$ is non-empty.  
    \end{assumption}
    
    \begin{remark}
        This modified assumption is indeed a generalization of the previous assumption, which can be seen by taking $a = 1, b=0$.
    \end{remark}
    
    This assumption then gives rise to a modified version of the Micro-LUCB algorithm which we state in Algorithm~\ref{alg:modified_micro_lucb}.
      \begin{algorithm}~\caption{Modified Micro-LUCB}
        \label{alg:modified_micro_lucb}
            Sample each source arm once.\;
            \For{$t = 1, 2, \ldots$}{
                $B_t = \argmax{\targetIndex \in [\numTargetArms]} \LCB{\sourceArm}{t}{\targetIndex}{\delta}$\;
                $C_t = \argmax{\targetIndex \in [\numTargetArms], \targetIndex \neq B_t} \UCB{\targetArm}{t}{\targetIndex}{\delta}$\;
                Choose any $I_t$ from $\Tilde{D}(B_t,\LCB{\sourceArm}{t}{B_t}{\delta}, \UCB{\sourceArm}{t}{B_t}{\delta})$\;
                Choose any $J_t$ from $\Tilde{D}(C_t,\LCB{\sourceArm}{t}{C_t}{\delta}, \UCB{\sourceArm}{t}{C_t}{\delta})$\;
                Observe $X_{t, 1} \sim S_{I_t}$ and $X_{t, 2} \sim S_{J_t}$\;
                Update $[\LCB{\sourceArm}{t}{I_t}{\delta}, \UCB{\sourceArm}{t}{I_t}{\delta}]$ and $[\LCB{\sourceArm}{t}{J_t}{\delta}, \UCB{\sourceArm}{t}{J_t}{\delta}]$\;
                \If{$\LCB{\targetArm}{t + 1}{B_t}{\delta} \geq \UCB{\targetArm}{t + 1}{C_t}{\delta}$}{
                    $\varSelectionRule \leftarrow B_t$\;
                    \Return $\varSelectionRule$\;
                }
            }

    \end{algorithm}
    
    It can be shown that only `diagonal' matrices satisfy the above assumption.
    We demonstrate this in the case $A \in \bbR_{\geq 0}^{2\times2}$ through the following proposition:
    \begin{proposition}\label{prop:micro-lucb-bad}
        Let $A \in \bbR^{2\times2}_{\geq 0}$. Suppose $A$ satisfies Assumption~\ref{ass:modified-micro-lucb-assumption}, then for $i = 1, 2$, either $A_{i1} = 0$ or $A_{i2}=0$.
    \end{proposition}
    % The proof of this proposition can be found at the end of this section.
    \begin{proof}
    % [Proof of Proposition~\ref{prop:micro-lucb-bad}]
        Let 
        \begin{equation*}
            A = \begin{bmatrix}
                    A_{11}       & A_{12} \\
                    A_{21}       & A_{22} \\
                \end{bmatrix},
        \end{equation*}
        where $A_{ij} \geq 0$. 
        Without loss of generality, we assume that $i = 1$ and $A_{11} \neq 0$, and we will demonstrate that this necessarily implies that $A_{12}=0$. 
        First, under Assumption~\ref{ass:modified-micro-lucb-assumption}, we know that 
            \begin{align}
                b_1 &\leq A_{11} u_1 + A_{12} u_2  - a_1 u_1, \\
                b_1 &\geq A_{11} v_1 + A_{12} v_2 - a_1 v_1.
            \end{align}
        Suppose we pick $v_1$ to satisfy 
        \begin{equation*}
            v_1 \geq \frac{A_{11}(u_1 - v_1) + A_{12}(u_2, v_2)}{a_1} + u_1.
        \end{equation*}
        Some straightforward algebra shows that 
        \begin{equation*}
            A_{11} u_1 + A_{12} u_2 - a_1 u_1 \leq A_{11} v_1 + A_{12} v_2 - a_1 v_1.
        \end{equation*}
        The above inequality then implies that
        \begin{equation*}
            b_1 \leq A_{11} u_1 + A_{12} u_2 - a_1 u_1 \leq A_{11} v_1 + A_{12} v_2 - a_1 v_1 \leq b_1,
        \end{equation*}
        which is only possible when 
        \begin{equation}\label{eq:prop_eq_1}
            A_{11} u_1 + A_{12} u_2 - a_1 u_1 = A_{11} v_1 + A_{12} v_2 - a_1 v_1.
        \end{equation}
        To see this is a contradiction, we rearrange equation~\eqref{eq:prop_eq_1} and observe that the following must hold for all $u \leq v$:
        \begin{equation*}
            A_{12}(v_2 - u_2) = (A_{11} - a_1)(u_1 - v_1).
        \end{equation*}
        However, this is cannot hold for all $u \leq v$ unless $A_{12} = (A_{11} - a_1) = 0$.
        This implies that $A_{12} = 0$.
        Therefore, $A_{12} = 0$, as desired. 
        (The same argument can be repeated to show that if $A_{12} \neq 0$, we must have $A_{11} = 0$). 
        % Therefore, we have shown that $A_{i1} = 0$ or $A_{i2}=0$ as desired.
    \end{proof}
    \newpage

\section{Proofs of Results}\label{sec:proofs}
    This section contains the proofs for the results given in Section~\ref{sec:results}.

    \subsection{Miscellaneous Results}
    Our analyses rely on the events that the means of the source and target arms stay within their respective confidence sequences. 
    Formally, we define this `good event', $\varGoodEvent$ as follows
    \begin{align}
        \varSourceGoodEvent &\coloneqq \bigcap_{t \in \bbN} \bigcap_{\sourceIndex \in [\numSourceArms]} \braket{\sourceMean_\sourceIndex \in \CI{\sourceArm}{t}{\sourceIndex}{\delta}}\label{def:good-source-event}, \\
        \varTargetGoodEvent &\coloneqq \bigcap_{t \in \bbN} \bigcap_{\targetIndex \in [\numTargetArms]} \braket{\targetMean_\targetIndex \in \CI{\targetArm}{t}{\targetIndex}{\delta}}\label{def:good-target-event}, \\
        \varGoodEvent &\coloneqq \varSourceGoodEvent \bigcap \varTargetGoodEvent\label{def:good-event}.
    \end{align}
    
    If $\beta$ is chosen as to satisfy the condition in equation~\ref{eq:beta-condition}, then we can show that $\varGoodEvent$ occurs with probability at least than $1 - \delta$.
    \begin{lemma}\label{lem:good-event-prob}
        Assume $\beta$ is chosen to satisfy condition~\eqref{eq:beta-condition} so that 
        \begin{equation}
            \bbP\braket{\exists t \geq 1: \mu_\sourceIndex \not \in \CI{\sourceArm}{t}{\sourceIndex}{\delta}} \leq \delta.
        \end{equation}
        Then,
        \begin{equation}
            \bbP\braket{\varGoodEvent} \geq 1 - \delta,
        \end{equation}
        where $\varGoodEvent$ is defined as in equation~\eqref{def:good-event}.
    \end{lemma}
    % \todo{Edit this proof to work for additive functions}
    \begin{proof}
        The condition in equation~\eqref{eq:beta-condition} implies that $\bbP\braket{\varSourceGoodEvent} \geq 1 - \delta$. 
        To prove the result, we show that $\varSourceGoodEvent$ implies $\varTargetGoodEvent$ which directly implies that  $\bbP\braket{\varGoodEvent} = \bbP\braket{\varSourceGoodEvent} \geq 1 - \delta$. 
        To see this, we fix $\targetIndex \in [\numTargetArms]$ and observe that on the event $\varSourceGoodEvent$
        \begin{align*}
            &\LCB{\targetArm}{\round}{\targetIndex}{\delta} =  \sum_{\sourceIndex \in [\numSourceArms]} \min_{m_\sourceIndex \in \CI{\sourceArm}{\round}{\sourceIndex}{\delta}} \varTransferFunction_{\targetIndex, \sourceIndex}(m_\sourceIndex) \leq \sum_{\sourceIndex \in [\numSourceArms]} f_{\targetIndex, \sourceIndex}(\sourceMean_\sourceIndex) = \targetMean_\targetIndex, \\
            &\UCB{\targetArm}{\round}{\targetIndex}{\delta} = \sum_{\sourceIndex \in [\numSourceArms]} \max_{m_\sourceIndex \in \CI{\sourceArm}{\round}{\sourceIndex}{\delta}} f_{\targetIndex, \sourceIndex}(m_\sourceIndex) \geq \sum_{\sourceIndex \in [\numSourceArms]} f_{\targetIndex, \sourceIndex}(\sourceMean_\sourceIndex) = \targetMean_\targetIndex,
        \end{align*}
        so that $\LCB{\targetArm}{t}{\targetIndex}{\delta} \leq \targetMean_\targetIndex \leq \UCB{\targetArm}{t}{\targetIndex}{\delta}$. Since $\targetIndex$ is arbitrary, the above result holds for all $\targetIndex \in [\numTargetArms]$. We have just shown that $\varSourceGoodEvent$ implies $\varTargetGoodEvent$ so that $\bbP(\varGoodEvent) = \bbP(\varSourceGoodEvent) \geq 1 - \delta$ as desired.
    \end{proof}
   
    We now use this result to prove Theorem~\ref{thm:correctness} which concerns the correctness of Algorithm~\ref{alg:transfer-lucb}.
   \begin{proof}[Proof of Theorem~\ref{thm:correctness}]
       We observe that by Lemma~\ref{lem:good-event-prob}, the event $\goodEvent$ occurs with probability at least $1 - \delta$. In particular, this implies that for each target arm, $\targetIndex$, and for every round, $\round$, we have that $\targetMean_\targetIndex \in \CI{\targetArm}{\round}{\targetIndex}{\delta}$.
       Suppose that the stopping condition is met and recall that we have set $\targetIndex = 1$ to be an optimal target arm.
       Then, if $B_t$ is an optimal target arm, the algorithm clearly returns an $\epsilon$-optimal arm.
       Next, suppose that $B_t$ is not an optimal target arm.
       In this case, we observe that
       \begin{equation*}
           \targetMean_{B_t} + \epsilon \geq \LCB{\targetArm}{\round}{B_t}{\delta} \geq \UCB{\targetArm}{\round}{C_t}{\delta} \geq \UCB{\targetArm}{\round}{1}{\delta} \geq \targetMean_1 = \max_{\targetIndex \in [\numTargetArms]} \targetMean_\targetIndex,
       \end{equation*}
       which implies that $B_t$ is $\epsilon$-optimal and thus proves the correctness of our algorithm, as desired.
   \end{proof}
    
    \subsection{Results for Additive Transfer Functions}\label{subsec:additive-proofs}
    For the readers convenience, before presenting the proof of Theorem~\ref{thm:tlucb-sample-complexity}, we briefly review the notation introduced in Section~\ref{sec:results}.
    % \todo{Current results only work for $\epsilon > 0$. Modify to work with $\epsilon \geq 0$.}
    We let 
    \begin{equation}
        \ecomplexityLength{\sourceIndex}{\targetIndex}{\round}{x} \coloneqq \max_{m \in [x, x + 2\beta(t, \delta)]} \varTransferFunction_{\targetIndex, \sourceIndex}(m) - \min_{m \in [x, x + 2\beta(\round, \delta)]} \varTransferFunction_{\targetIndex, \sourceIndex}(m)
    \end{equation}
    to represent the length of target arm $\targetIndex$'s confidence interval contributed by source arm $\sourceIndex$ when $\LCB{\sourceArm}{\round}{\sourceIndex}{\delta} = x$. Next, we define 
    \begin{equation}
                 \varStoppingTime_{\targetIndex, \sourceIndex} = \min\braket{t \in \bbN: \sup_{x \in [\sourceMean_\sourceIndex - 2\beta(\round, \delta), \sourceMean_\sourceIndex]}  \ecomplexityLength{\sourceIndex}{\targetIndex}{\round}{x} < \frac{\max\braket{|\varImaginaryTargetArm - \targetMean_\targetIndex|, \epsilon/2}}{\varSparsity_\targetIndex}},
    \end{equation}
    and
    \begin{equation}
                \varStoppingTime_{\sourceIndex} = \max_{\targetIndex \in [\numTargetArms]} \varStoppingTime_{\targetIndex, \sourceIndex}.
    \end{equation}
    
    \begin{lemma}
    Let $(P_\round, Q_\round) \in \braket{(B_\round, I_\round), (C_\round, J_\round)}$. On the good event $\goodEvent$, if $N_{Q_\round}(\round) \geq \tau_{P_\round, Q_\round}$, then 
    \begin{equation}
        \UCB{\targetArm}{\round}{P_\round}{\delta} - \LCB{\targetArm}{\round}{P_\round}{\delta} \leq \max\braket{|\varImaginaryTargetArm - \targetMean_{P_\round}|, \epsilon/2}.
    \end{equation}
    \end{lemma}
    \begin{proof}
        % \todo{Verify that all the indices are correct. I found some places where I used $P_t$ instead of $Q_t$ and $\targetIndex$ instead of $\sourceIndex$ etc.}
        Since we are on the good event, it must be true that $\sourceMean_{Q_\round} \geq \LCB{\sourceArm}{\round}{Q_\round}{\delta} \geq \sourceMean_{Q_\round} - 2\beta(N_{Q_\round}(\round), \delta)$. Therefore, the definition of $\varStoppingTime_{P_\round, Q_\round}$ implies that if $N_{Q_\round}(\round) \geq \varStoppingTime_{P_\round, Q_\round}$, then 
        \begin{align*}
            \algLength{Q_\round}{P_\round}{\round} &= \max_{m \in \CI{\sourceArm}{\round}{Q_\round}{\delta}} \varTransferFunction_{P_\round, Q_\round}(m) -  \min_{m \in \CI{\sourceArm}{\round}{Q_\round}{\delta}} \varTransferFunction_{P_\round, Q_\round}(m) \\&\leq \frac{\max\braket{|\varImaginaryTargetArm - \targetMean_{P_\round}|, \epsilon/2}}{\varSparsity_{P_\round}},
        \end{align*}
        where the inequality follows by the definition of $\varStoppingRule_{P_t, Q_t}$.
        Additionally, by the definition of the selection rule, we observe that for all $\sourceIndex \in [\numSourceArms]$,
        \begin{equation*}
            \ealgLength{Q_\round}{P_\round}{\round} \geq \ealgLength{\sourceIndex}{P_\round}{\round}.
        \end{equation*}
        Therefore, the following inequalities must hold
        \begin{align*}
            \UCB{\targetArm}{\round}{P_\round}{\delta} - \LCB{\targetArm}{\round}{P_\round}{\delta} &= \sum_{\sourceIndex \in [\numSourceArms]} \ealgLength{\sourceIndex}{P_\round}{\round} \\
            &\leq \sum_{\sourceIndex \in [\numSourceArms]} \ealgLength{Q_\round}{P_\round}{\round} \\
            &= \varSparsity_{P_\round} \ealgLength{Q_\round}{P_\round}{\round} \\
            &\leq \max\braket{|\varImaginaryTargetArm - \targetMean_{P_\round}|, \epsilon/2},
        \end{align*}
        which gives us the desired result.
    \end{proof}
    
            \begin{lemma}\label{lem:target_confidence_interval_overlap_1}
            Recall that $\varImaginaryTargetArm = \frac{\targetMean_1 + \targetMean_2}{2}$. On the good event $\varGoodEvent$ defined in equation~\eqref{def:good-event}, if the algorithm has not terminated, then there exists $P_t \in \{B_t, C_t\}$ such that 
                    \begin{equation}
                        \varMaxTerm{P_t} \leq |\CI{\targetArm}{t}{P_t}{\delta}|.
                    \end{equation}
        \end{lemma}
        
        \begin{proof}
            We will split the proof into two cases which encompass all possible scenarios. The first case is when $|\varImaginaryTargetArm - \targetMean_{P_t}| \geq \frac{\epsilon}{2}$, and the other case is when $\frac{\epsilon}{2} \geq |\varImaginaryTargetArm - \targetMean_{P_t}|$. 
            
            \paragraph{Case 1} We start off by showing that $|\CI{\targetArm}{t}{P_t}{\delta}| \geq |\varImaginaryTargetArm - \targetMean_{P_t}|$. Here we assume that $|\varImaginaryTargetArm - \targetMean_{P_t}| \geq \frac{\epsilon}{2}$.  Suppose for the purpose of contradiction that $\varImaginaryTargetArm \not \in \CI{\targetArm}{t}{P_t}{\delta}$. If this is the case, then one of the following four statements must be true:
            \begin{enumerate}
                \item $\varImaginaryTargetArm < \LCB{\targetArm}{t}{B_t}{\delta}$ and $\varImaginaryTargetArm < \LCB{\targetArm}{t}{C_t}{\delta}$. However, on $\varGoodEvent$, the only arm which can have a lower confidence bound greater than $\varImaginaryTargetArm$ is arm $1$.
                
                \item $\varImaginaryTargetArm > \UCB{\targetArm}{t}{B_t}{\delta}$ and $\varImaginaryTargetArm > \UCB{\targetArm}{t}{C_t}{\delta}$. However, on $\varGoodEvent$, the upper confidence bound of arm $1$, and hence the upper confidence bound of $B_t$, must be greater than $\varImaginaryTargetArm$.
                
                \item $\varImaginaryTargetArm > \UCB{\targetArm}{t}{B_t}{\delta}$ and $\varImaginaryTargetArm < \LCB{\targetArm}{t}{C_t}{\delta}$. However, on $\varGoodEvent$, the upper confidence bound of arm $1$, and hence the upper confidence bound of $B_t$, must be greater than $\varImaginaryTargetArm$.
                
                \item $\varImaginaryTargetArm < \LCB{\targetArm}{t}{B_t}{\delta}$ and $\varImaginaryTargetArm > \UCB{\targetArm}{t}{C_t}{\delta}$. This would imply that he algorithm has terminated, which by assumption, is false.
            \end{enumerate}
            Therefore, by our initial assumption we observe that there exists $P_t \in \{B_t, C_t\}$ satisfying $\varMaxTerm{P_t} \leq |\CI{\targetArm}{t}{P_t}{\delta}|$.
            
            \paragraph{Case 2} Here we show that exists a $P_t \in \{B_t, C_t\}$ such that $|\CI{\targetArm}{t}{P_t}{\delta}| \geq \frac{\epsilon}{2}$. For this case, we assume  that $\frac{\epsilon}{2} \geq |\varImaginaryTargetArm - \targetMean_{P_t}|$.
            % Suppose for the purpose of contradiction that $|\CI{\targetArm}{t}{B_t}{\delta}| \leq \epsilon/2$ and $|\CI{\targetArm}{t}{C_t}{\delta}| \leq \epsilon/2$. 
            By the definition of the stopping rule, we know that 
                \begin{equation}\label{eq:stopping_condition_lemma}
                    \LCB{\targetArm}{t}{B_t}{\delta} < \UCB{\targetArm}{t}{C_t}{\delta} - \epsilon.
                \end{equation}
            We observe that $|\CI{\targetArm}{t}{B_t}{\delta}| + |\CI{\targetArm}{t}{C_t}{\delta}| > \UCB{\targetArm}{t}{C_t}{\delta} - \LCB{\targetArm}{t}{B_t}{\delta}$.
            Then rearranging equation~\eqref{eq:stopping_condition_lemma} yields 
                \begin{align*}
                    \epsilon    & < \UCB{\targetArm}{t}{C_t}{\delta} - \LCB{\targetArm}{t}{B_t}{\delta} \\
                                & < |\CI{\targetArm}{t}{B_t}{\delta}| + |\CI{\targetArm}{t}{C_t}{\delta}|.
                \end{align*}
            % However, this is a contradiction if both $|\CI{\targetArm}{t}{B_t}{\delta}| \leq \epsilon/2$ and $|\CI{\targetArm}{t}{C_t}{\delta}| \leq \epsilon/2$. 
            Therefore, by our initial assumption we observe that there exists $P_t \in \{B_t, C_t\}$ satisfying $\varMaxTerm{P_t} \leq |\CI{\targetArm}{t}{P_t}{\delta}|$.
            
            We have thus shown that, in both cases, there exists $P_t \in \{B_t, C_t\}$ satisfying $P_t \in \{B_t, C_t\}$ satisfying $\varMaxTerm{P_t} \leq |\CI{\targetArm}{t}{P_t}{\delta}|$, which proves the desired result.
            % We have shown that every situation in which $c \not \in [L_{P_t}(t), U_{P_t}(t)]$ is not possible. Therefor it must be true that $c \in [L_{P_t}(t), U_{P_t}(t)]$, giving the desired result.
        \end{proof}
        
    \begin{lemma}
        On the good event, $\goodEvent$, if the algorithm has not stopped, then there exists a pair $(P_\round, Q_\round) \in \braket{(B_\round, I_\round), (C_\round, J_\round)}$ such that $N_{Q_\round}(\round) < \varStoppingTime_{P_\round, Q_\round}$ 
    \end{lemma}
    \begin{proof}
    % \todo{Double check this. The end, in particular.}
     By Lemma \ref{lem:target_confidence_interval_overlap_1} we know that
                \begin{align*}
                    \varMaxTerm{\targetIndex} &\leq |\CI{\targetArm}{t}{P_t}{\delta}| \\
                    &= \uosum{\sourceIndex = 1}{\numSourceArms} \ealgLength{\sourceIndex}{P_\round}{\round}.
                \end{align*}
                
    By applying the pigeonhole principle, we see that there must exist at least one $\sourceIndex' \in [\numSourceArms]$ such that 
            \begin{equation*}
                \ealgLength{\sourceIndex'}{P_\round}{\round} \geq \frac{\varMaxTerm{P_t}}{s_{P_t}}.
            \end{equation*}
    Then, by applying the definition of the selection rule, and the fact that on the good event $\sourceMean_{Q_\round} \geq \LCB{\sourceArm}{\round}{Q_\round}{\delta} \geq  \sourceMean_{Q_\round} - 2\beta(N_{Q_\round}(\round), \delta)$, we observe that
    \begin{align*}
        \sup_{x \in [\sourceMean_\sourceIndex - 2\beta(N_{Q_\round}(\round), \delta), \sourceMean_\sourceIndex]} \ecomplexityLength{Q_\round}{P_\round}{\round}{x}
                                                &\geq \ealgLength{Q_\round}{P_\round}{\round} \\ 
                                                &\geq \ealgLength{\sourceIndex'}{P_\round}{\round} \\
                                                &\geq \frac{\varMaxTerm{P_t}}{s_{P_t}}.
    \end{align*}
    This implies that
    \begin{align*}
        N_{Q_\round}(\round)    &\leq \min\braket{t \in \bbN: \sup_{x \in [\sourceMean_{Q_t} - 2\beta(\round, \delta), \sourceMean_{Q_t}]}  \ecomplexityLength{Q_t}{P_t}{\round}{x} < \frac{\max\braket{|\varImaginaryTargetArm - \targetMean_{P_t}|, \epsilon/2}}{\varSparsity_{P_t}}} \\
                                &= \varStoppingTime_{P_\round, Q_\round},
    \end{align*}
    as desired.
    \end{proof}
    
     We are now ready to prove Theorem~\ref{thm:tlucb-sample-complexity}.
        \begin{proof}[Proof of Theorem~\ref{thm:tlucb-sample-complexity}]
            We have
                \begin{align*}
                                    \varStoppingTime            
                                    &= \uosum{\round = 1}{\infty} \indicator{t \leq \tau} \\
                                    &\leq \uosum{\round = 1}{\infty} \indicator{\exists (P_t, Q_t) \in \braket{(B_t, I_t), (C_t, J_t)}: N_{Q_\round}(\round) \leq \varStoppingTime_{P_\round, Q_\round}} \\
                                    &\leq \uosum{\sourceIndex \in [\numSourceArms]}{}\uosum{\round = 1}{\infty} \indicator{\sourceIndex \in \braket{I_\round, J_\round}} \cdot \indicator{N_\sourceIndex(\round) \leq \max_{\targetIndex \in [\numTargetArms]} \varStoppingTime_{\targetIndex, \sourceIndex}} \\
                                    &=\uosum{\sourceIndex \in [\numSourceArms]}{}\uosum{\round = 1}{\infty} \indicator{\sourceIndex \in \braket{I_\round, J_\round}} \cdot \indicator{N_\sourceIndex(\round) \leq \varStoppingTime_{\sourceIndex}} \\
                                    &\leq \uosum{\sourceIndex\in[\numSourceArms]}{} \varStoppingTime_{\sourceIndex},
                \end{align*}
        which proves the desired result.
        \end{proof}

    \begin{proof}[Proof of Corollary~\ref{cor:property-testing}]
        Suppose $\sourceIndex \in \armSet$ since we otherwise don't need to sample source arm $\sourceIndex$ to determine if $\armSet$ is the optimal target arm. From equation~\eqref{eq:property-testing-length} we see that $L(\sourceIndex, \propertySet, \round, x) = \infty$ unless the confidence interval for $\sourceMean_\sourceIndex$ is a subset of $\propertySet_\sourceIndex$ or $\propertySet_\sourceIndex^c$.
        We consider two cases.
        \paragraph{Case 1.} Suppose that $\mu_\sourceIndex \in \propertySet_\sourceIndex$. Then, we require the confidence interval is a subset of $\propertySet_\sourceIndex$. For this to be true, it is easy to see that we require $2 \beta(t, \delta) \leq \inf_{x \in \propertySet^c_\sourceIndex} |x - \sourceMean_\sourceIndex| = \Delta_{\propertySet_\sourceIndex}(\sourceMean_\sourceIndex)$.
        
        \paragraph{Case 2.} Suppose that $\sourceMean_\sourceIndex \not \in \propertySet_\sourceIndex$. Then a similar argument shows that we require $2 \beta(t, \delta) \leq \inf_{x \in \propertySet_\sourceIndex} |x - \sourceMean_\sourceIndex| = \Delta_{\propertySet_\sourceIndex}(\sourceMean_\sourceIndex)$.
        
        In conclusion, we see that if $\sourceIndex \in \armSet$, then $\varStoppingTime_{\armSet, \sourceIndex} = \inf\{ \round \in \mathbb N: \beta(\round, \delta) \leq \frac{\Delta_{\propertySet_\sourceIndex}(\sourceMean_\sourceIndex)}{2}\}$.
        Applying Theorem 16 of \cite{kaufmann2017monte} gives the desired result.  
    \end{proof}
    
    \begin{proof}[Proof of Corollary~\ref{eq:cor-linear-sample-complexity}]
        We observe that since $\ecomplexityLength{\sourceIndex}{\targetIndex}{\round}{x} = 2|\varLinearTransformation_{\targetIndex, \sourceIndex}|\beta(\round, \delta)$ we have
        \begin{equation*}
            \varStoppingTime_{\targetIndex, \sourceIndex} = \min\left\{\round \in \bbN : \beta(\round, \delta) \leq \frac{\max\braket{|\varImaginaryTargetArm - \targetMean_\targetIndex|, \epsilon/2}}{\varSparsity_\targetIndex |\varLinearTransformation_{\targetIndex, \sourceIndex}}\right\}.
        \end{equation*}
        Applying Theorem 16 of \cite{kaufmann2017monte} and taking the max over target arms gives the desired result
        
    \end{proof}

        \subsection{Results for Instantiations}\label{app:instantiations}
        \paragraph{TopK}
        For TopK, we observe that $\ecomplexityLength{\sourceIndex}{\targetIndex}{\round}{x} = 2\beta(t, \delta)$ which is independent of $x$. Additionally, we note that for all $a \in [\numTargetArms]$, $\varSparsity_\targetIndex = K$. Therefore, for a fixed $\targetIndex$, we have $\varStoppingTime_{\targetIndex, \sourceIndex} = \min\{\round \in \mathbb N: \beta(\round, \delta) \leq \frac{|\bar{\targetMean}_{1, 2} - \targetMean_\targetIndex|}{K}\}$. Applying Theorem 16 of \cite{kaufmann2017monte} we have 
        \begin{equation*}
            \varStoppingTime_{\targetIndex, \sourceIndex} \leq O\left(\frac{K^2}{(\overline{\targetMean}_{1, 2} - \targetMean_\targetIndex)^2} \log \frac{1}{\delta}\right).
        \end{equation*}
        Next, we observe that 
        \begin{align*}
            \max_{a \in \targetIndex} \varStoppingTime_{\targetIndex, \sourceIndex} &\leq O\left(\frac{K^2}{(\bar{\mu} - \mu_i)^2}\right),
        \end{align*}
        which can be seen by choosing $\nu_a$ to be the target arm which contains 
        \begin{enumerate}
            \item $\mu_1, \ldots, \mu_{K - 1}, \mu_i$ if $i > K$;
            \item $\mu_1, \ldots, \mu_{K}$ if $i \leq K$.
        \end{enumerate}

        \paragraph{Thresholding Bandits} 
        Since this a special case of the property testing problem, we see that Corollary~\ref{cor:property-testing} implies that $\Delta_{\propertySet_\sourceIndex}(\sourceMean_\sourceIndex) = (\sourceMean_\sourceIndex - \threshold)$, which gives the desired result.
            
\end{document}